\documentclass[a4paper]{article}

\usepackage{natbib}
\setcitestyle{authoryear,open={(},close={)}} %Citation-related commands

\usepackage[english]{babel}
\usepackage{amsmath,amsthm,amssymb,amsfonts}
\usepackage{color}
\usepackage{graphicx}
\usepackage{subcaption}
\usepackage{url}
\usepackage{hyperref}
\usepackage{authblk}
\usepackage{booktabs}

\usepackage{algorithm}
\usepackage{algorithmicx}
\usepackage{algpseudocode}

\newtheorem{theorem}{Theorem}
\newtheorem{lemma}{Lemma}

\newcommand{\N}{\mathbb{N}}
\newcommand{\set}[1]{\left\{ #1 \right\}}

\renewcommand{\L}{\mathcal{L}}
\newcommand{\D}{\mathcal{D}}
\newcommand{\E}{\mathbb{E}}

\newcommand{\heap}{\textsc{Heap}}
\renewcommand{\succ}{\textsc{Succ}}
\newcommand{\seen}{\textsc{Seen}}
\newcommand{\query}{\textbf{Query}}

\begin{document}

\title{Scaling Neural Program Synthesis with Distribution-based Search}

\author[1,2]{Nathana{\"e}l Fijalkow}
\author[1]{Guillaume Lagarde}
\author[1]{Th{\'e}o Matricon}
\author[3]{Kevin Ellis}
\author[4]{Pierre Ohlmann}
\author[5]{Akarsh Potta}
\affil[1]{CNRS, LaBRI and Universit{\'e} de Bordeaux, France}
\affil[2]{The Alan Turing Institute of data science, United Kingdom}
\affil[3]{Cornell University, United States}
\affil[4]{University of Paris, France}
\affil[5]{Indian Institute of Technology Bombay, India}
\date{}

\maketitle

\begin{abstract}
We consider the problem of automatically constructing computer programs from input-output examples.
We investigate how to augment probabilistic and neural program synthesis methods with new search algorithms,
proposing a framework called distribution-based search.
Within this framework, we introduce two new search algorithms:
\textsc{Heap Search}, an enumerative method, and \textsc{SQRT Sampling}, a probabilistic method.
We prove certain optimality guarantees for both methods,
show how they integrate with probabilistic and neural techniques,
and demonstrate how they can operate at scale across parallel compute environments.
Collectively these findings offer theoretical and applied studies of search algorithms for program synthesis that integrate with recent developments in machine-learned program synthesizers.
\end{abstract}

\section{Introduction}

Writing software is tedious, error-prone, and accessible only to a small share of the population
-- yet coding grows increasingly important as the digital world plays larger and larger roles in peoples' lives.
Program synthesis seeks to make coding more reliable and accessible by developing methods for automatically constructing programs~\cite{Gulwani2017}.
For example, the FlashFill system~\cite{Gulwani2017} in Microsoft Excel makes coding more accessible by allowing nontechnical users to synthesize spreadsheet programs by giving input-output examples,
while the TF-coder system~\cite{TFCoder} seeks to make coding neural networks more reliable by synthesizing TensorFlow code from input-outputs.
Where these systems have been most successful is when they pair a specialized 
\emph{domain-specific language} (DSL) to a domain-specific search algorithm for synthesizing programs in that DSL.
A recent trend -- both in industry~\cite{Kalyan2018} and academia~\cite{Devlin2017} -- is to employ machine learning methods to learn to quickly search for a program in the DSL~\cite{Balog2017,Devlin2017,Lee2018,Zhang2018,Polosukhin2018,Kalyan2018,ZoharW18,chen2018execution}.
Many such recent works have explored engineering better neural networks for guiding program search, effectively by training the network to act as a language model over source code that conditions on input-outputs~\cite{Polosukhin2018}.
Here, we `pop up' a level and instead ask: given a neural net that probabilistically generates source code, how can we most efficiently deploy that model in order to find a program consistent with some input-outputs?
This concern arises because program synthesis requires solving a hard combinatorial search problem (exploring a possibly infinite space of programs),
so taming this combinatorial explosion makes the difference between a practically useful system, and a system which cannot scale to anything but the most trivial of programs.

At a high-level the approaches we develop in this work follow a 2-stage pipeline: in the first stage a learned model predicts probabilistic weights,
and in the second stage a symbolic search algorithm uses those weights to explore the space of source code.
Our contributions target the second stage of this pipeline,
and we focus on theoretical analysis of sampling-based search algorithms, new search algorithms based on neurally-informed enumeration, and empirical evaluations showing that recent neural program synthesizers can compose well with our methods.

This 2-stage pipelined approach has several benefits. The first is that the cost of querying the neural network is usually very small compared to the cost of combinatorial search, yet in practice the neural model learns to provide rough-and-ready probabilistic predictions to guide the search. A second benefit is that even if the probabilistic predictions are inaccurate, our methods maintain soundness and completeness (but may take longer to run). Another appeal is that it can be naturally combined with other classical approaches for program synthesis.

\paragraph*{Our contributions:}
\begin{itemize}
	\item A theoretical framework called distribution-based search for evaluating and comparing search algorithms
	in the context of machine-learned predictions. 
	\item Two new search algorithms: \textsc{Heap Search}, an enumerative method, and \textsc{SQRT Sampling}, a probabilistic method. We prove a number of theoretical results about them, in particular that they are both loss optimal.
	\item A method for running any search algorithm across parallel compute environments.
\end{itemize}
We perform an empirical evaluation of existing and new search algorithms, showing how the new methods integrate with probabilistic and neural techniques.

\begin{figure*}[t]
  \centering
  \includegraphics[width=\textwidth]{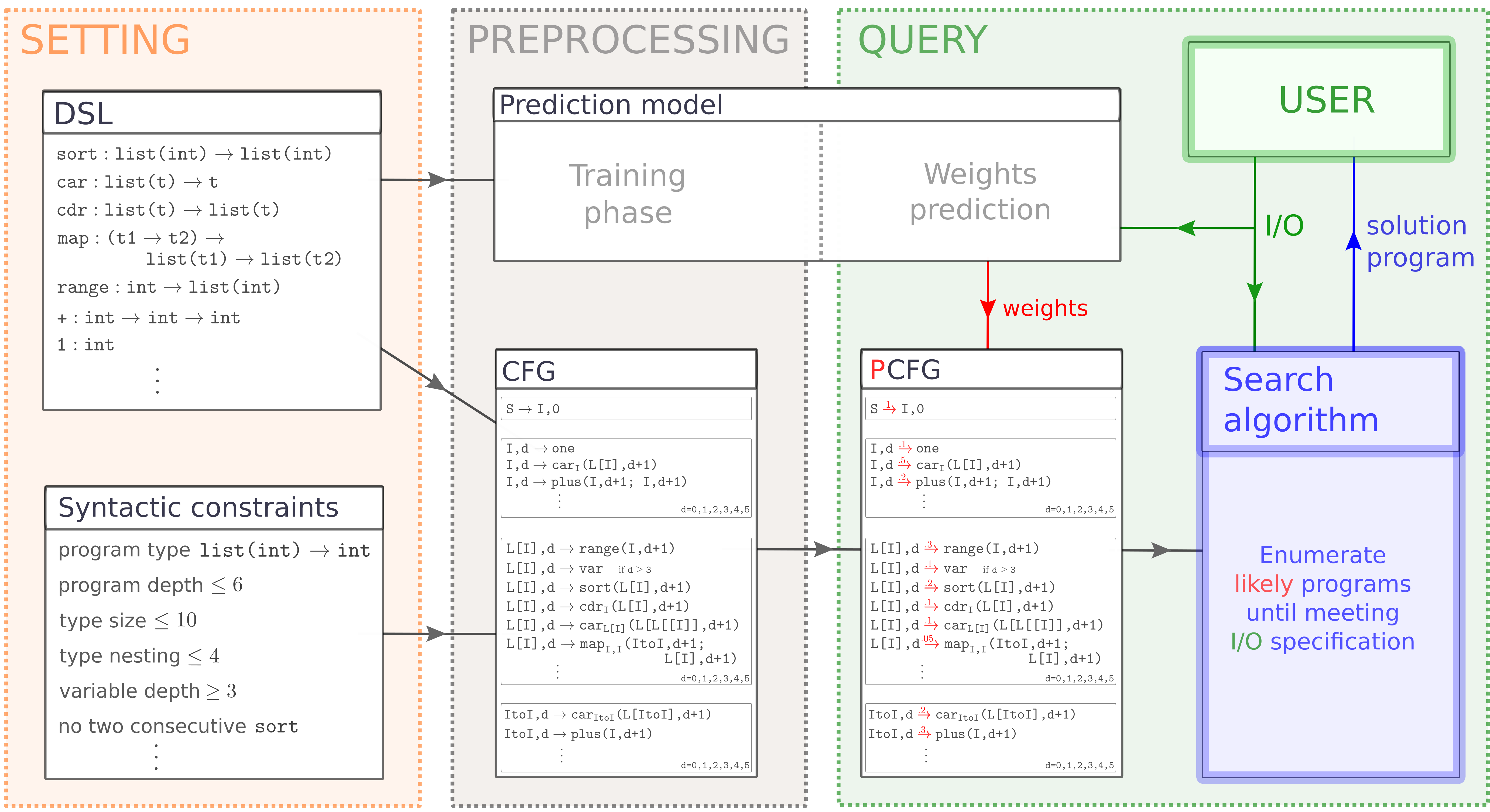}
  \caption{Pipeline for neural predictions for syntax guided program synthesis.}
  \label{fig:example_dsl_cfg}
\end{figure*}

\section{Distribution-based search}

We work within the syntax guided program synthesis (SyGuS) framework introduced by~\cite{AlurBJMRSSSTU13}, see also~\cite{AlurSFS18}. 
In this setting, the DSL is given by a set of primitives together with their (possibly polymorphic) types and semantics.

We describe the machine learning pipeline for program synthesis, illustrated in 
Figure~\ref{fig:example_dsl_cfg} on a toy DSL describing integer list manipulating programs.

The compilation phase constructs a context-free grammar (CFG) from the DSL together with a set of syntactic constraints.
The CFG may incorporate important information about the program being generated,
such as the $n$ last primitives (encompassing $n$-gram models) or semantic information (\textit{e.g.} non-zero integer, sorted list).

A prediction model (typically a neural network) takes as inputs a set of I/O and outputs a probabilistic labelling of the CFG, inducing a probabilistic context-free grammar (PCFG). 
The network is trained so that most likely programs (with respect to the PCFG) are the most likely to be solutions, meaning map the inputs to corresponding outputs.

We refer to Appendix~\ref{sec:appendix_framework} for an in-depth technical discussion on program representations
and on the compilation phase.
In this work we focus on the search phase and start with defining a theoretical framework for analysing search algorithms.

\vskip1em
The PCFG obtained through the predictions of the neural network defines a probabilistic distribution $\D$ over programs.
We make the theoretical assumption that the program we are looking for is actually sampled from $\D$,
and construct algorithms searching through programs which find programs sampled from $\D$ as quickly as possible.
Formally, the goal is to minimise the expected number of programs the algorithm outputs before finding the right program.

We write $A(n)$ for the $n$\textsuperscript{th} program chosen by the algorithm $A$; since $A$ may be a randomised algorithm $A(n)$ is a random variable.
The performance $\L(A,\D)$ of the algorithm $A$, which we call its loss, is the expected number of tries it makes before finding $x$:
\[
\L(A,\D) = \E_{x \sim \D} \left[ \inf \set{n \in \N : A(n) = x} \right].
\]
An algorithm $A^*$ is `loss optimal' if $\L(A^*,\D) = \inf_A \L(A,\D)$.
Let us state a simple fact: an algorithm is loss optimal 
if it generates each program once and in non-increasing order of probabilities.
Depending on $\D$ constructing an efficient loss optimal algorithm may be challenging,
pointing to a trade off between quantity and quality: 
is it worth outputting a lot of possibly unlikely programs quickly,
or rather invest more resources into outputting fewer but more likely programs?

\paragraph*{An example.}
To illustrate the definitions let us consider the distribution $\D$ over the natural numbers such that $\D(n) = \frac{1}{2^{n+1}}$; it is generated by the following PCFG:
\[
S \to^{.5} f(S) \quad ; \quad S \to^{.5} x,
\]
when identifying $n$ with the program $f^n(x)$.
Let us analyse a few algorithms.
\begin{itemize}
	\item The algorithm $A_1$ enumerates in a deterministic fashion the natural numbers starting from $0$: 
	$A_1(n) = n$.
	Then $\L(A_1,\D) = \sum_{n \ge 0} \frac{n+1}{2^{n+1}} = 2$.
	This enumeration algorithm $A_1$ is loss optimal.

	\item The algorithm~$A_2$ samples the natural numbers using the distribution~$\D$.
	For $n \ge 0$, the value of $\E \left[ \inf \set{n' : A_2(n') = n} \right]$ is $2^{n+1}$: 
	this is the expectation of the geometric distribution with parameter $\frac{1}{2^{n+1}}$.
	Then $\L(A_2,\D) = \sum_{n \ge 0} \frac{2^{n+1}}{2^{n+1}} = \infty$.
	Hence the naive sampling algorithm using $\D$ has infinite loss.

	\item The algorithm $A_3$ samples the natural numbers using a distribution that we call $\sqrt{\D}$
	defined\footnote{ For the normalisation factor, note that $\sum_{n \ge 0} \frac{1}{2^{\frac{n+1}{2}}} = 1 + \sqrt{2}$.} 
	by $\sqrt{\D}(n) = \frac{1}{1 + \sqrt{2}} \frac{1}{2^{\frac{n+1}{2}}}$.

	For $n \ge 0$, the value of $\E \left[ \inf \set{n' : A_3(n') = n} \right]$ is $(1 + \sqrt{2}) 2^{\frac{n+1}{2}}$: 
	this is the expectation of the geometric distribution with parameter $\frac{1}{1 + \sqrt{2}} \frac{1}{2^{\frac{n+1}{2}}}$.
	Then 
	\[
	\begin{array}{lll}
	\L(A_3,\D) & = & \sum_{n \ge 0} \frac{(1 + \sqrt{2}) 2^{\frac{n+1}{2}}}{2^{n+1}} \\
	 & = & (1 + \sqrt{2}) \sum_{n \ge 0} \frac{1}{2^{\frac{n+1}{2}}} \\
	 & = & (1 + \sqrt{2})^2 \approx 5.83.
	\end{array}
	\]
	As we will prove in a more general statement (Theorem~\ref{thm:optimal_sampling}), 
	the algorithm $A_3$ is loss optimal among sampling algorithms. 
	Suprisingly it is not much worse than the loss optimal algorithm, yet offers many advantages: 
	it is much easier to implement, and requires no memory at all. 
	Last but not least in the case of PCFG it can be implemented using a new probabilistic labelling of the PCFG inducing~$\D$.
\end{itemize}

The next two sections are devoted to the two natural classes of algorithms for distribution-based search: \emph{enumeration} and \emph{sampling}. We then study how they can run at scale accross parallel compute environments.

\section{Enumerative methods and \\ the \textsc{Heap Search} algorithm}

A number of enumerative methods have been investigated in previous works~\cite{Menon2013, Balog2017, Feng2018, ZoharW18}.
They proceed in a top-down fashion, and can be understood as ways of exploring the tree of leftmost derivations of the grammar
as illustrated in Figure~\ref{fig:enumerative_methods}.

\begin{figure}
  \centering
  \includegraphics[width=0.99\columnwidth]{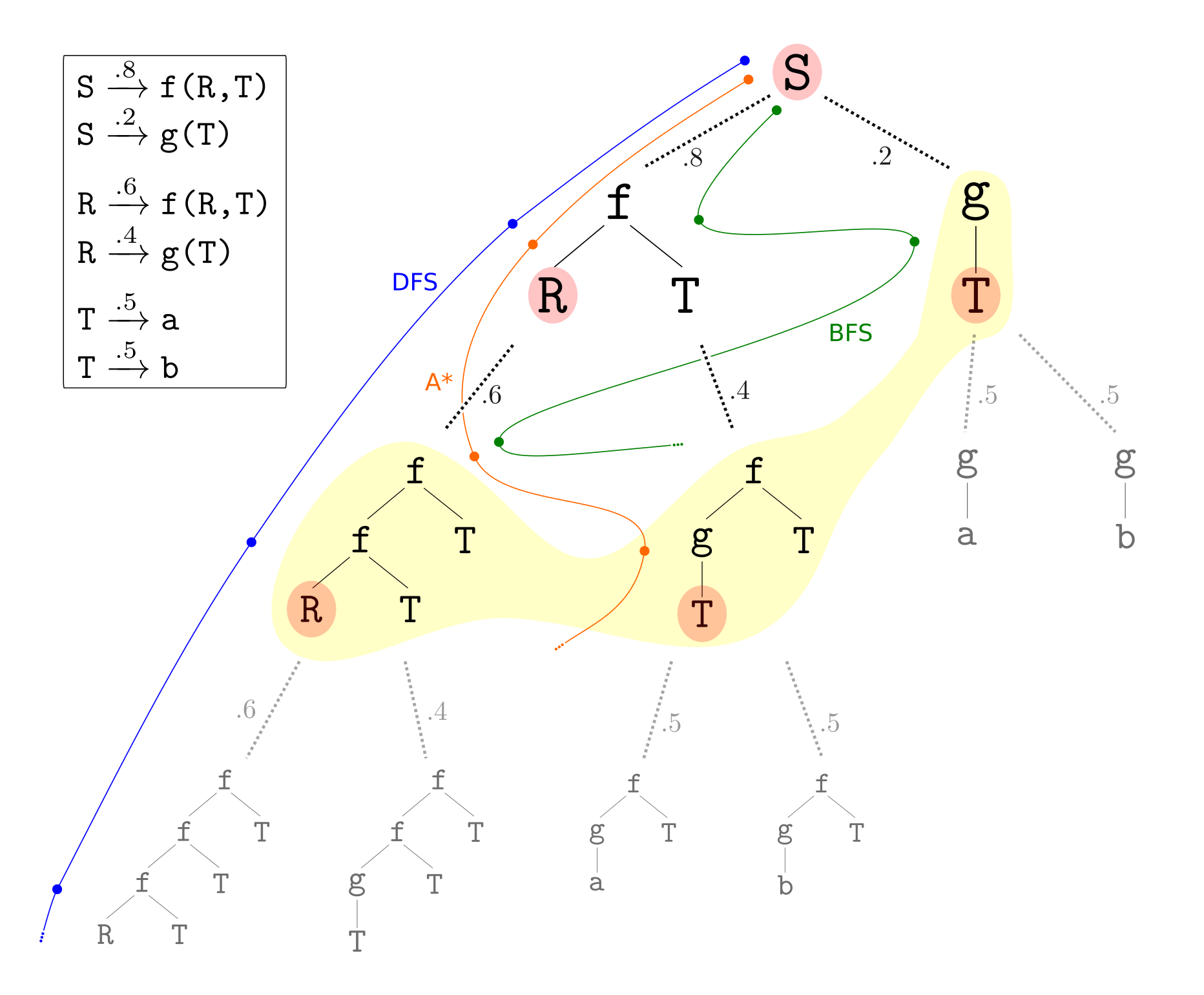}
  \caption{Illustration of the tree of leftmost derivations.}
  \label{fig:enumerative_methods}
\end{figure}

We present a new efficient and loss optimal algorithm called \textsc{Heap Search} and following a bottom-up approach.
It uses a data structure based on heaps to efficiently enumerate all programs in non-increasing order of the
probabilities in the PCFG.

Let us write $\D$ for the distribution induced by a PCFG.
For a program $x$, we say that $x'$ is the `successor of $x$' if it is the most likely program after $x$, 
meaning $\D(x) > \D(x')$ and there are no programs $x''$ such that $\D(x) > \D(x'') > \D(x')$.
For a non-terminal $T$ in the grammar, the `successor of $x$ from $T$' is the most likely program
after $x$ among those generated from $T$. We define `predecessor of
$x$' and `predecessor of $x$ from $T$' in a similar way.

We create a procedure $\query(T,x)$ which for any program $x$ generated from a non-terminal $T$ outputs the successor of $x$ from $T$. 
Note that once this is defined, the algorithm performs successive calls to $\query(S,x)$ with $S$ the initial non-terminal and $x$ the latest generated program, yielding all programs in non-increasing order of the probabilities.

To explain how $\query(T,x)$ works, we first define the data structure.
For each non-terminal $T$ we have a hash table $\succ_T$ which stores the successors of all previously seen programs generated from $T$, and a heap $\heap_T$ which contains candidate programs, ordered by non-increasing probability.
The key invariant is the following: 
the successor of $T$ from $x$ has either already been computed, hence is in $\succ_T$,
or is the maximum program in $\heap_T$.
This means that implementing $\query(T,x)$ is very simple: it checks whether the successor has already been computed and returns it in that case, and if not it pops the heap.
The difficulty is in maintaining the invariant; for this we need to add a number of candidate programs to the heaps.
They are obtained by substituting one argument from the returned successor, and propagate this recursively to the corresponding non-terminals. 

\begin{theorem}
\label{thm:heap_search}
The \textsc{Heap Search} algorithm is loss optimal: it enumerates every program exactly once and in non-increasing order of probabilities.
\end{theorem}
We refer to Appendix~\ref{sec:appendix_heap_search} for a complete description of the algorithm with a pseudocode,
a proof of Theorem~\ref{thm:heap_search}, and a computational complexity analysis.

\textsc{Heap Search} is related to $A^*$ from~\cite{Feng2018}: they are both loss optimal algorithms,
meaning that they both enumerate programs in non-increasing order of probabilities.
As we will see in the experiments \textsc{Heap Search} is better in two aspects:
it is faster, and it is bottom-up, implying that program evaluations can be computed along with the programs and avoiding evaluating twice the same (sub)program.

\section{Sampling methods and\\ the \textsc{SQRT Sampling} algorithm}

A \emph{sampling algorithm} takes random samples from a distribution $\D'$; what is both a strength and a weakness is that a sampling algorithm is memoryless: a weakness because the algorithm does not remember the previous draws, which means that it may draw them again, but also a strength because it uses very little space and can be very easily implemented.

In the case of sampling algorithms,  we identify algorithms with distributions.
The following theorem shows a dichotomy: 
either there exists a loss optimal sampling algorithm among sampling algorithms,
and then it is characterised as the `square root' of the distribution,
or all sampling algorithms have infinite loss.

\begin{theorem}\label{thm:optimal_sampling}
Let $\D$ a distribution over a set $X$.
If $\sum_{x \in X} \sqrt{\D(x)} < \infty$, the distribution $\sqrt{\D}$ defined by 
\[
\sqrt{\D}(x) = \frac{\sqrt{\D(x)}}{\sum_{y \in X} \sqrt{\D(y)}}
\]
is loss optimal among all sampling algorithms.
If $\sum_{x \in X} \sqrt{\D(x)} = \infty$, for all sampling algorithms $\D'$ we have $\L(\D',\D) = \infty$.
\end{theorem}

\begin{proof}
Let $\D'$ be a distribution. 
For an element $x$, the expectation of the number of tries for $\D'$ to draw $x$ is $\frac{1}{\D'(x)}$:
this is the expectation of success for the geometric distribution with parameter $\D'(x)$.
It follows that
\[
\L(\D',\D) = \E_{x \sim D} \left[ \frac{1}{\D'(x)} \right] = \sum_{x \in X} \frac{\D(x)}{\D'(x)}.
\]

Let us assume that $\sum_{x \in X} \sqrt{\D(x)} < \infty$.
Thanks to Cauchy-Schwarz inequality we have:
\[
\begin{array}{lll}
\left( \sum_{x \in X} \sqrt{\D(x)} \right)^2 
& = & \left( \sum_{x \in X} \sqrt{\frac{\D(x)}{\D'(x)}} \sqrt{\D'(x)} \right)^2 \\
& \le & \left( \sum_{x \in X} \frac{\D(x)}{\D'(x)} \right) \cdot \underbrace{\left( \sum_{x \in X} \D'(x) \right)}_{= 1} \\
& = & \sum_{x \in X} \frac{\D(x)}{\D'(x)}.
\end{array}
\]
We note that $\L(\sqrt{\D},\D) = \left( \sum_{x \in X} \sqrt{\D(x)} \right)^2$,
so the previous inequality reads $\L(\D',\D) \ge \L(\sqrt{\D},\D)$. 
Thus $\sqrt{\D}$ is loss optimal among sampling algorithms,
and if it is not defined, then for any $\D'$ we have $\L(\D',\D) = \infty$.
\end{proof}

Theorem~\ref{thm:optimal_sampling} characterises the loss optimal sampling algorithm, but does not explain how to implement it.
The following result answers that question.

\begin{theorem}\label{thm:optimal_sampling_pcfg}
If $\D$ is defined by a PCFG and $\sqrt{\D}$ is well defined,
then we can effectively construct a PCFG defining~$\sqrt{\D}$.
\end{theorem}
The PCFG for $\sqrt{\D}$ is obtained from the PCFG for $\D$ by taking the square root of each transition probability,
and then globally renormalising. 
Details of this procedure can be found in Appendix~\ref{sec:appendix_SQRT_sampling}.

\section{Parallel implementations}

Harnessing parallel compute environments is necessary for scalable, future-proof search algorithms,
because combinatorial search bottlenecks on compute,
and both the present and likely future of massive compute is a parallel one.
Accordingly, we have taken care to design and evaluate extensions of our algorithms which can metabolize these compute resources through multiprocessing.

We introduce a new algorithm called the \emph{grammar splitter}, which partitions a PCFG into a balanced family of $k$ sub-PCFGs.
Each of the $k$ threads is assigned a sub-PCFG and simply runs a search algorithm on it.
Two key advantages of our approach are that any search algorithm can be used in this very simple parallel architecture,
and that the theoretical gain of using $k$ threads is linear in $k$.
The output of the grammar splitter is illustrated in Figure~\ref{fig:grammar_splitter}: the white PCFG is split into $4$ sub-PCFGs.

\begin{figure}[t]
  \centering
  \includegraphics[width=0.99\columnwidth]{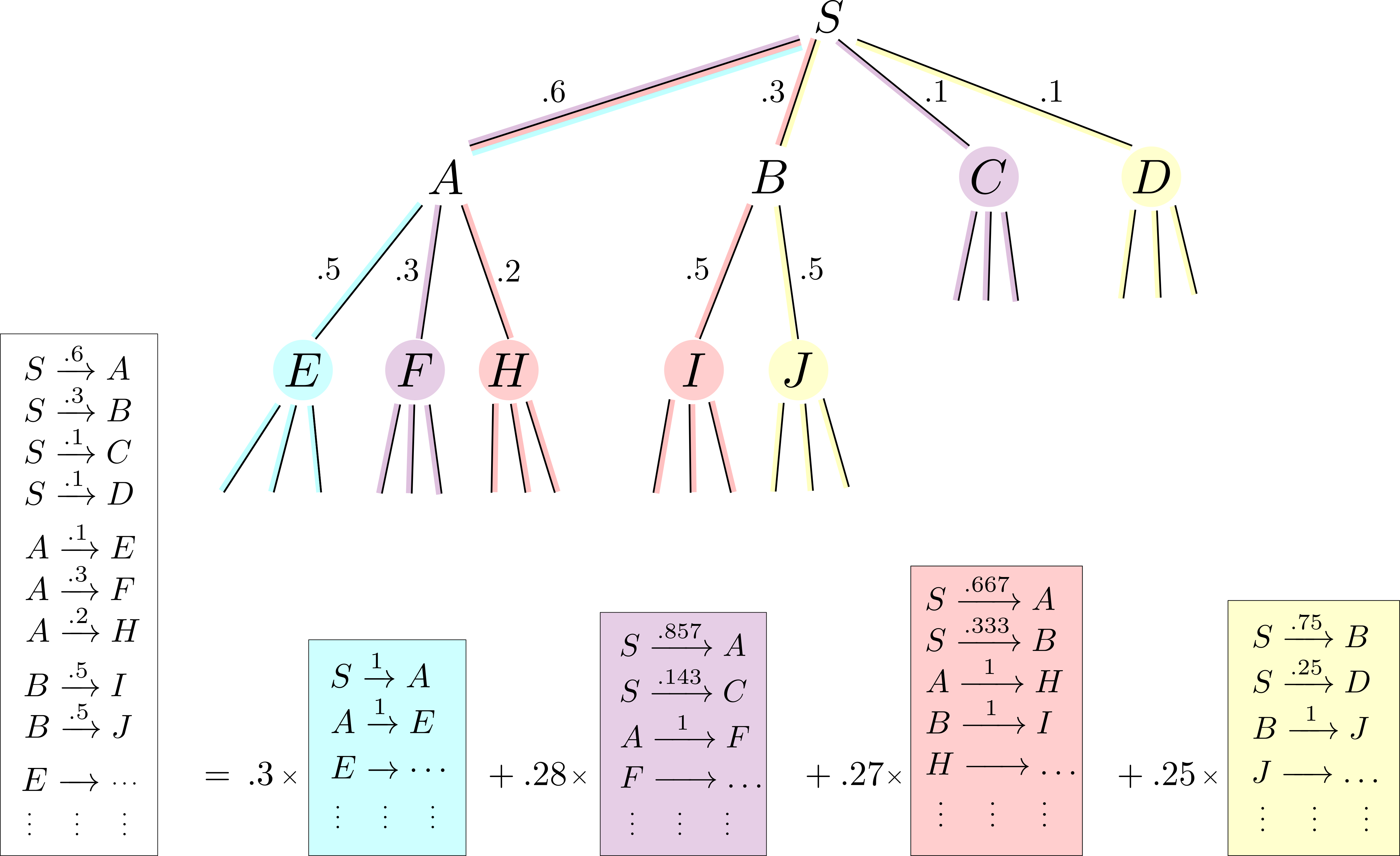}
  \caption{The grammar splitter: a balanced partition with quality $\alpha = \frac{.3}{.25} = 1.2$.}
  \label{fig:grammar_splitter}
\end{figure}

The two crucial properties of the grammar splitter are:
\begin{itemize}
	\item the space of programs is partitioned into $k$ subspaces. 
	This implies that the threads do not carry out redundant work and that all programs are generated,
	\item the $k$ program subspaces are balanced, meaning that their mass probabilities are (approximately) equal. 
	This implies that all threads contribute equally to the search effort.
\end{itemize}

A split is a collection of partial programs, for instance \verb+map (* 2) HOLE+ and \verb!fold + HOLE HOLE!,
it induces a PCFG.
A set of $k$ incomparable splits yields a partition of the PCFG.
Let us write $\alpha$ for the quality of a partition, defined as the ratio between the maximum and the minimum probability mass of a split.
We are looking for a balanced partition, \textit{i.e.} one for which the quality $\alpha$ is close to $1$.

Our algorithm finds a balanced partition through a hill climbing process:
at each point the algorithm either looks for an improving swap or a refinement.
In the first case, the action of an improving swap is to transfer a partial program from one split to another,
and its goal is to lower the quality coefficient.
In the second case, we consider the partial program with maximal probability in a split
and refine it: for example \verb+map (* 2) HOLE+ could be replaced by 
\verb+map (* 2) var0+ and \verb+map (* 2) (filter HOLE HOLE)+.

\begin{figure*}
    \centering
    \subfloat[\centering Cumulative probability against time in log-scale]{{\includegraphics[scale=0.8]{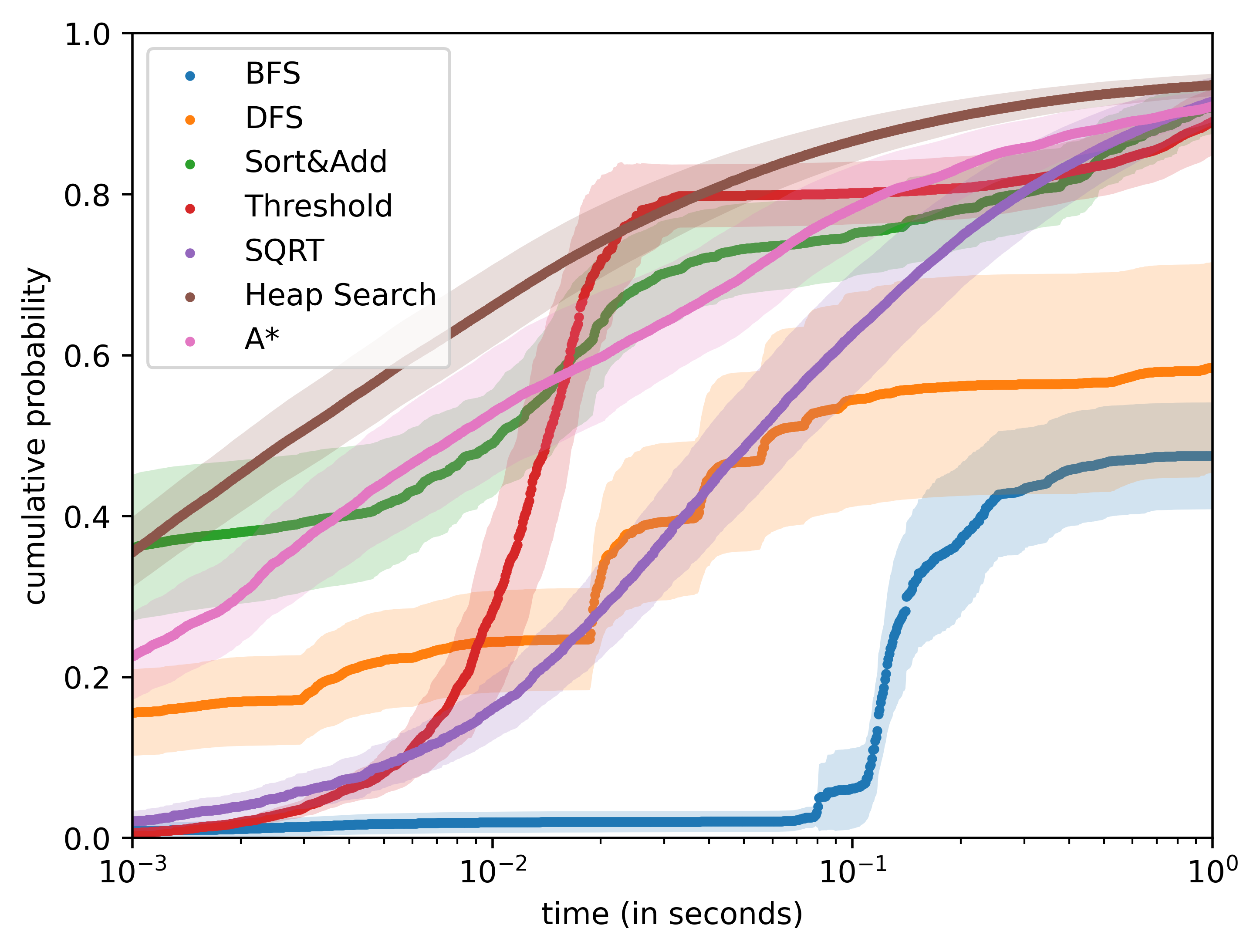}}}
    \qquad
    \subfloat[\centering Cumulative probability against number of programs output in log-scale]{{
\includegraphics[scale=0.8]{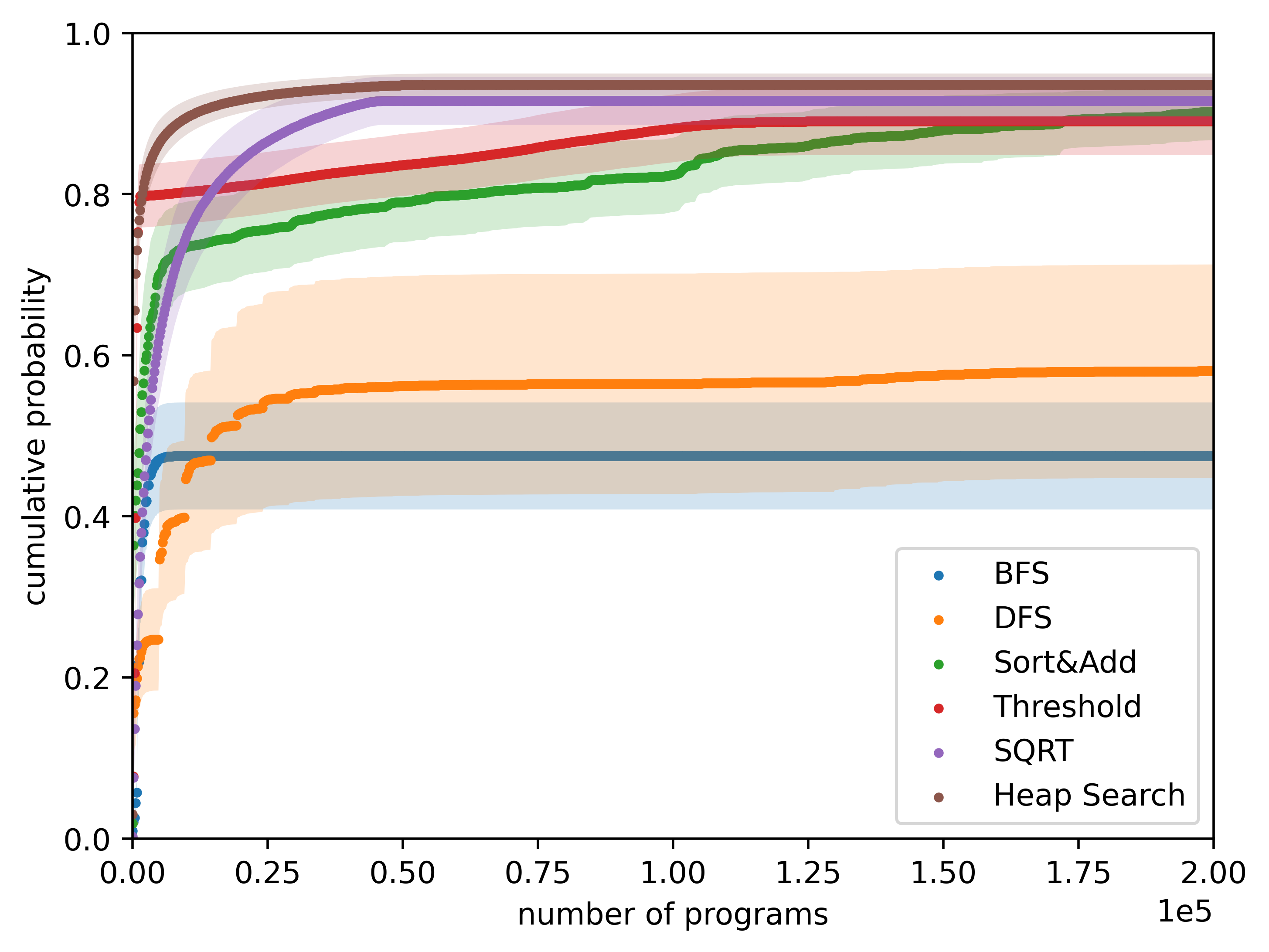}}}
    \caption{Comparing all search algorithms on random PCFGs}
    \label{fig:search_randomPCFGs}
\end{figure*}

\section{Experiments}

We study a range of search algorithms -- both our new ones and prior work -- across list processing and string manipulation domains, with the goal of answering the following questions:
\begin{itemize}
\item \textsc{Heap Search} and $A^*$ are both loss optimal enumerative algorithms; beyond these theoretical guarantees, how do the two algorithms compare in practice?
\item How effective are our search algorithms for solving complex program synthesis benchmarks using neural guidance?
\item How do our algorithms scale with parallel compute? 
\end{itemize}

We use a generic program synthesizer written from scratch in Python
(Appendix~\ref{sec:appendix_framework}),
studying random PCFGs (more controlled) and machine-learned PCFGs (more naturalistic).

We report results on DSLs from DeepCoder~\cite{Balog2017} and DreamCoder~\cite{EllisWNSMHCST21}.
Both target the classic program synthesis challenge of integer list processing programs, but with different properties.
DeepCoder's DSL is larger and more specialized, with around $40$ high-level primitives,  and does not use polymorphic types,
while DreamCoder's is smaller and more generic, with basic functional programming primitives such as map, fold, unfold, car, cons, and cdr, etc., for a total of around $20$ primitives.
Both DSLs are compiled into a CFG with minimal syntactic constraints generating programs of depth $6$.

The search algorithms under consideration are:
\begin{itemize}
\item \textsc{Threshold} from~\cite{Menon2013}: iterative-deepening-search, where the threshold that is iteratively deepened is a bound on program description length (i.e. negative log probability),
\item \textsc{Sort and add} from~\cite{Balog2017}: an inner loop of depth-first-search, with an outer loop that sorts productions by probability and runs depth-first-search with the top $k$ productions for increasing values of $k$,
\item $A^*$ from~\cite{Feng2018}: best-first-search on the graph of (log probabilities of) tree derivations,
\item \textsc{Beam search} from~\cite{Zhang2018}: breadth-first-search with bounded width that is iteratively increased.
\end{itemize}
As well as our new algorithms: \textsc{Heap Search} and \textsc{SQRT Sampling}.
We refer to Appendix~\ref{sec:appendix_enumerative} for a description of the algorithms and their implementations.

Our implementation of \textsc{SQRT Sampling} uses the Alias method~\cite{VOSE},
which is an efficient data structure for sampling from a categorical distribution.
We associate to each non-terminal an Alias table, reducing the task of sampling a derivation rule with $n$ choices
to sampling uniformly in $[1,n]$ and in $[0,1]$.

All algorithms have been reimplemented and optimised in the codebase
to provide a fair and uniform comparison.
We also report on parallel implementations using our grammar splitter.

\subsection{Random PCFGs}

In this first set of experiments we run all search algorithms on random PCFGs until a timeout,
and compare the number of programs they output and the cumulative probability of all programs output.

To obtain random PCFGs from the CFGs we sample a probabilistic labeling with an exponential decrease (this is justified by the fact that machine-learned PCFGs feature exponential decrease in transition probabilities).
In this experiment the initialization cost of each algorithm is ignored.
The results presented here are averaged over $50$ samples of random PCFGs, the solid lines represent the average
and a lighter color indicates the standard deviation.
Details on the sampling procedure can be found in Appendix~\ref{sec:appendix_experiments}.
 
Figure~\ref{fig:search_randomPCFGs} shows the results for all algorithms in a non-parallel implementation.
On the lhs we see that \textsc{Heap Search} (almost always) has the highest cumulative probability against both time and number of distinct programs.
Note that since $A^*$ and \textsc{Heap Search} enumerate the same programs in the same order they produce the same curve in the rhs of Figure~\ref{fig:search_randomPCFGs} so we did not include $A^*$. 

To compare $A^*$ and \textsc{Heap Search} we refer to Figure~\ref{fig:comparison_heap_search_a_star},
showing that \textsc{Heap Search} generates $2.35$ times more programs than $A^*$, consistently over time.
The larger variations for $A^*$ are due to the manipulation of a single heap of growing size, requiring frequent memory reallocations.
The difference in performance can be explained by the fact that $A^*$ uses a single heap for storing past computations,
while \textsc{Heap Search} distributes this information in a family of connected heaps and hash tables.

\begin{figure}
  \centering
  \includegraphics[width=0.99\columnwidth]{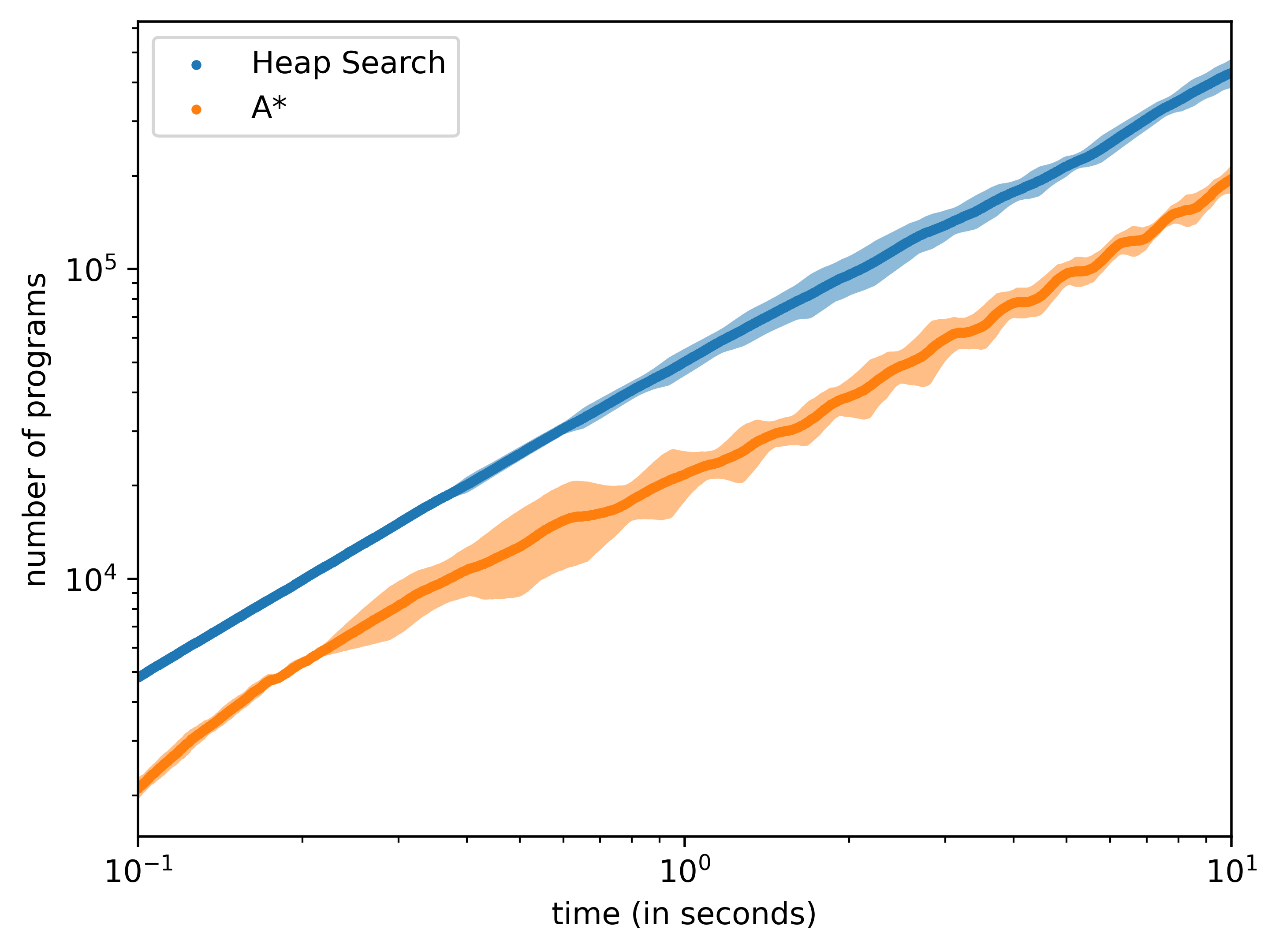}
  \caption{Comparing \textsc{Heap Search} and $A^*$}
  \label{fig:comparison_heap_search_a_star}
\end{figure}

We then turn to parallel implementation and perform the same experiments using a variable number of CPUs for \textsc{Heap Search} and \textsc{SQRT Sampling} using the grammar splitter. 
We do not report on a baseline parallel implementation of \textsc{SQRT Sampling} which would simply sample using the same PCFG on multiple CPUs. Indeed this naive approach performs poorly in comparison, since thanks to the grammar splitter two CPUs cannot generate the same program.

The results are shown in Figure~\ref{fig:parallel}, where we count programs with repetitions.
We see that for \textsc{SQRT Sampling} the scale-up is linear, and it is mostly linear for \textsc{Heap Search} with an acceleration from the 0.2s mark.
This acceleration can be explained in two ways: 
first, each sub-PCFG is shallower since it is split thus it is faster to enumerate program from it,
second, once all the successors have been computed \textsc{Heap Search} is a simple lookup table.
At the end of the experiment, \textsc{SQRT Sampling} has generated 2.8 times more programs with 6 CPUs than with 2 CPUs, 
whereas \textsc{Heap Search} has generated 7.6 times more programs with 6 CPUs than with 2 CPUs.

This experiment suggests that the grammar splitter enables us to scale our search on multiple CPUs with a linear speed up in the number of CPUs.

\begin{figure}
  \centering
  \includegraphics[width=0.99\columnwidth]{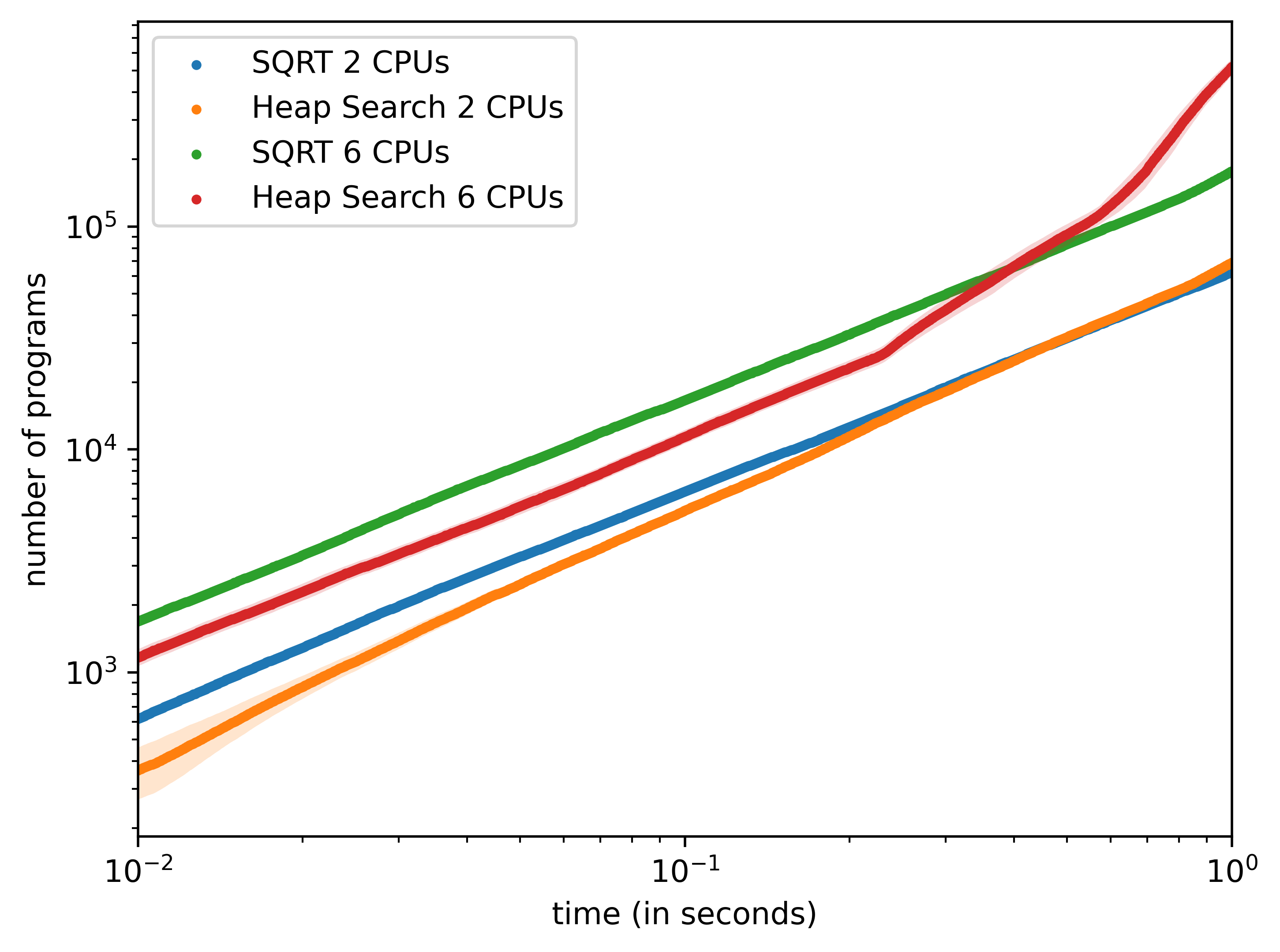}
  \caption{Parallel implementations of \textsc{Heap Search} and \textsc{SQRT Sampling} using the grammar splitter}
  \label{fig:parallel}
\end{figure}

\subsection{Machine-learned PCFGs}

In this second set of experiments we consider the benchmark suites of hand-picked problems and sets of I/O.
We extracted 218 problems from DreamCoder's dataset~\citep{EllisWNSMHCST21}.
(The experiments can be easily replicated on DeepCoder's dataset~\citep{Balog2017} but we do not report on the results here.)

We train a neural network to make predictions from a set of I/O.
Our neural network is composed of a one layer GRU~\cite{cho-etal-2014-learning} and a 3-layer MLP with sigmoid activation functions,
and trained on synthetic data generated from the DSL.
The details of the architecture and the training can be found in Appendix~\ref{sec:appendix_experiments}.
Our network architecture induces some restrictions, for instance the types of the programs must be \verb+int list -> int list+;
we removed tasks that did not fit our restrictions and obtained a filtered dataset of 137 tasks.
For each task we run every search algorithm on the PCFG induced by the neural predictions with a timeout of $100$s and a maximum of $1M$ programs. Unlike in the previous experiments the initialization costs of algorithms are not ignored.

\begin{figure}[th]
  \centering
  \includegraphics[width=0.99\columnwidth]{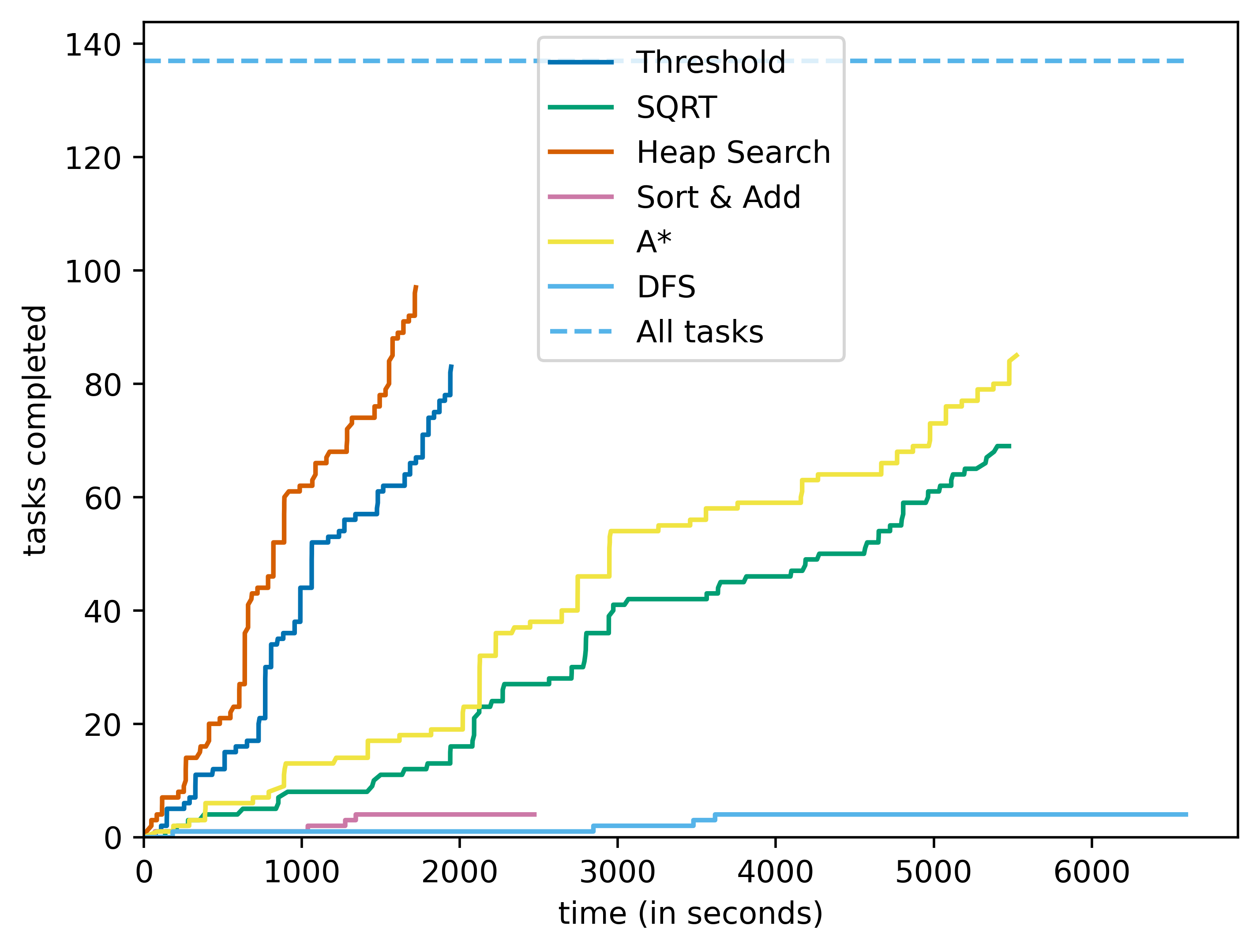}
  \caption{Comparing all search algorithms on the DreamCoder reduced dataset with machine-learned PCFGs}
  \label{fig:machine_learned_dreamcoder}
\end{figure}

Figure~\ref{fig:machine_learned_dreamcoder} shows the number of tasks solved within a time budget.
\textsc{Heap Search} solves the largest number of tasks for any time budget, and in particular 97 tasks out of 137 before timeout.
The comparison between \textsc{Threshold} and $A^*$ is insightful: 
$A^*$ solves a bit more tasks than \textsc{Threshold} (85 vs 83) but in twice the time.
\textsc{SQRT Sampling} performs just a bit worse than $A^*$ despite being a sampling algorithm.

Table~\ref{tab:number_programs} shows for each algorithm how many programs were generated per second.
Recall that \textsc{Heap Search} and $A^*$ generate the same programs in the same order.
Overall in these experiments \textsc{Heap Search} is $6$ times faster than~$A^*$.

Since \textsc{Heap Search} follows a bottom-up approach we save on program evaluations in two ways:
partial programs are evaluated along the search, and the results are cached.
On the other hand $A^*$ is a top-down enumeration method so every new program has to be evaluated from scratch.

It is interesting to compare the rates of \textsc{SQRT Sampling} and $A^*$: 
although \textsc{SQRT Sampling} generates over two times more programs, their overall performances are similar.
This can be explained in two ways: \textsc{SQRT Sampling} may sample the same programs many times, 
while $A^*$ enumerates each program once and starts with the most promising ones according to the predictions.

\begin{table}
  \caption{Number of programs generated}
  \label{tab:number_programs}
  \centering
  \begin{tabular}{cc}
    \toprule
    Algorithm & Number of programs generated \\
    \midrule
    \textsc{Heap Search}       & 38735 prog/s\\
    \textsc{Threshold}     		 & 25381 prog/s\\
    \textsc{DFS}               & 20281  prog/s\\
    \textsc{SQRT Sampling}     & 14020 prog/s\\
    $A^*$                      & 6071  prog/s\\
    \bottomrule
  \end{tabular}
\end{table}

Finally, we want to see how \textsc{Heap Search} improves with the quality of the predictions. According to the properties of \textsc{Heap Search}, we expect that the better the predictions the faster tasks are solved since \textsc{Heap Search} is loss optimal.
In order to show this, we ran \textsc{Heap Search} on the reduced DreamCoder dataset with a timeout of 5 minutes and kept the tasks where a solution was found since some tasks cannot be solved with our DSL.
Then we trained a neural network on this new reduced set with the solutions found. 
At different epochs of the training we checked how fast and how many tasks \textsc{Heap Search} could solve with the predictions of the network with a timeout of 30 seconds and a limit of 1M programs.
We plot on Figure~\ref{fig:heapsearch_evol} the number of tasks solved with respect to total time used by \textsc{Heap Search} after different number of training epochs with the uniform PCFG as a baseline.
We clearly observe the expected outcome, as the number of training epochs grows and the neural networks learn to better predict the solutions, \textsc{Heap Search} dramatically decreases the time required to solve the tasks. While this may be true for every algorithm, loss optimal algorithms like \textsc{Heap Search} benefit the most from it.
  
\begin{figure}[th]
  \centering
  \includegraphics[width=0.99\columnwidth]{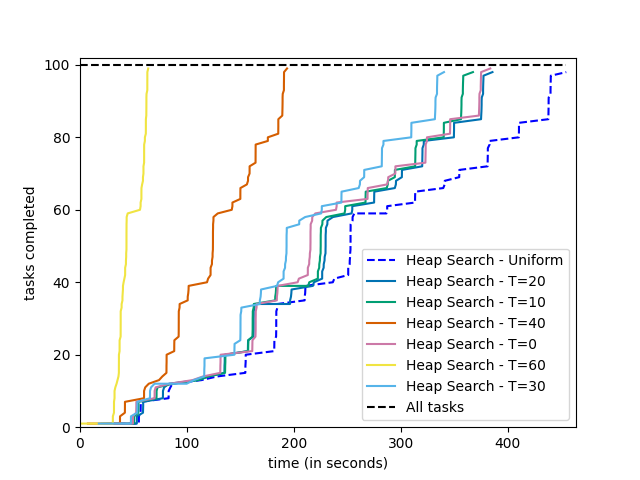}
  \caption{Evolution of time and tasks solved by \textsc{Heap Search} at different epochs of the training of a machine-learning PCFG predictor on a reduced DreamCoder dataset}
  \label{fig:heapsearch_evol}
\end{figure}

\section{Discussion}

\subsection{Related work}

The idea of guiding program search via probabilities is an old one~\cite{Solomonoff1989ASF} but which recently has fast become a standard practice in the AI and program synthesis communities.
To the best of our knowledge, practical program synthesizers that learn to use probabilistic predictions originated in~\citep{Menon2013},
which was first extended with deep learning in~\citep{Balog2017},
and such methods are now winning program synthesis competitions~\cite{Lee2018}.
To first approximation, such recent progress has drawn on advances in neural network architecture: e.g., early learned FlashFill-like systems~\cite{ParisottoMSLZK16} benefited from sophisticated attention mechanisms~\cite{Devlin2017},
and procedural planning programs~\cite{BunelHDSK18} benefited from feeding execution traces into the neural net~\cite{chen2018execution}.
While prior works have explored novel test-time strategies, from Sequential Monte Carlo~\cite{Ellis2019WriteEA} to ensembling methods~\cite{chen2018execution},
here we sought a systematic empirical/theoretical study of two different families of inference strategies, 
which we intend to mesh well with the larger body of work on neural program synthesis.
While the algorithms introduced here do not straightforwardly extend to neural autoregressive models (e.g. RobustFill~\cite{Devlin2017}),
methods such as \textsc{SQRT Sampling} in principle apply to this setting too.
We hope that our work here spurs the development of the right tricks to get the theoretical benefits of \textsc{SQRT Sampling} for these more flexible model classes, just as DeepCoder paved the way for RobustFill. 
Many extensions:~\cite{Lee2018,Zhang2018,Polosukhin2018,Kalyan2018,ZoharW18}. 

\citet{shi2020uniquerandomizer} introduced Unique Randomizer in order to sample without replacement:
it is a general technique effectively turning a sampling method into an enumerative one 
by updating the probabilistic weights along the search.
It is further improved through batching via Stochastic Beam Search~\cite{kool2019stochastic}.
It is possible to combine the \textsc{SQRT Sampling} algorithm with the Unique Randomizer and Stochastic Beam Search.
Our experiments did not yield interesting results in that direction, 
possibly because of memory allocation issues.
We leave for future work to optimise this approach.

\subsection{Contributions and Outlook}

Learning to synthesize programs is a canonical neural-symbolic learning problem:
training high capacity statistical models to guide the construction of richly structured combinatorial objects, such as programs.
Yet while the neural side of this problem has received much deserved attention, the symbolic component is sometimes taken for granted--after all, symbolic program synthesis has received decades of attention.
But the entrance of powerful neural networks for synthesizing programs forces us to reconsider how we deploy symbolic methods for program synthesis.
We have considered both systematic and stochastic methods, from both theoretical angles (obtaining guarantees) and also engineering perspectives (such as how to parallelize our new techniques).
We hope this work helps contribute to thinking through the symbolic search back-end in this more modern context.

\bibliography{bib}

\newpage
\appendix

\section{Machine-learned program synthesis}
\label{sec:appendix_framework}

\subsection{Program representation as trees}
\label{subsec:program_representation}

Programs are internally represented as trees as illustrated on the lhs of Figure~\ref{fig:tree_string}.
The program described by this tree is \verb+(take (head (rev l1)) (rev l2))+. 
We enforce a simple type system for programs consisting of basic types such as \verb+int, bool, string, float+
and two constructors: \verb+list+ and \verb+arrow+.
Each primitive in the DSL comes with a possibly polymorphic type, for instance \verb+take+ has the type 
\verb+take: int -> t0 list -> t0 list+, where \verb+t0+ can be instantiated to any type.
We only construct correctly typed programs using classical techniques from programming languages such as 
type inference and de Bruijn's indices for variables.

\begin{figure}[h]
  \centering
  \includegraphics[width=0.8\columnwidth]{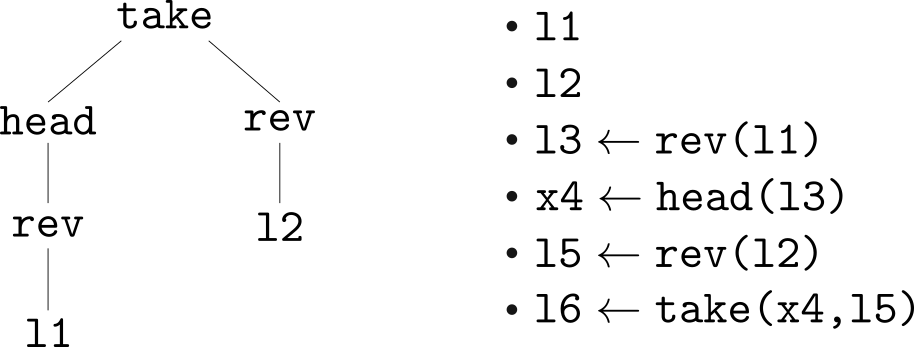}
  \caption{Tree and straight line representation of the same program.}
  \label{fig:tree_string}
\end{figure}

\subsection{Compilation phase: from domain specific language to grammar}
\label{subsec:compilation}

A domain specific language (DSL) is a set of primitives together with their type and semantics.
The compilation phase turns a DSL and a set of syntactic constraints into a CFG which generates the set of all correctly typed programs by following typing rules.
Syntactic constraints can be very diverse, they include maximum program depth, type request for the program, and other invariants to be satisfied.
The type request induces the initial non-terminal of the CFG:
for instance if the type request is 
\verb+int -> int list -> int list+
then the initial non-terminal carries the information that the program should output an \verb+int list+,
and that we can use one variable of type \verb+int+ and another one of type \verb+int list+.

Polymorphic types can be accommodated in two different ways: at compilation time or at generation time.
DreamCoder's implementation~\cite{EllisWNSMHCST21} 
uses polymorphic types in the CFGs, and performs type inference when generating programs.
In our tool we instantiate types during the compilation into CFGs, implying that they do not contain polymorphic types.
So when adding the primitive \verb+take: int -> t0 list -> t0 list+
we create a number of primitives for each type of bounded size, for instance

\verb+take[int]: int -> int list -> int list+.

This makes the CFGs larger, but considerably reduces the computation time for generating programs
by avoiding type inference algorithms.

The compilation phase consists in expanding typing rules as long as syntactic constraints are satisfied.
The CFGs are trimmed to keep their sizes manageable, removing non-productive and non-reachable non-terminals.

\section{Review of enumerative methods}
\label{sec:appendix_enumerative}

We describe the different search algorithms constructed in~\cite{Menon2013,Balog2017,Feng2018,Zhang2018},
which can be understood as ways of exploring the tree of leftmost derivations of the grammar.

\paragraph*{BFS and DFS for CFGs.}
The first algorithms we mention do not make use of the weights in the PCFG, they are enumeration algorithms for CFGs.
We consider the tree of all leftmost partial derivations starting from the initial non-terminal; each node of this tree is a partial program.
There are two classical algorithms for searching this tree. Both maintain a tree of partial derivations.
To resolve ambiguity, for each non-terminal we fix an (arbitrary) order on the derivation rules from this non-terminal.
The breadth-first search (BFS) expands the \emph{highest} node containing a non-terminal by applying the next derivation rule to the leftmost non-terminal in this node (next with respect to the order on derivation rules from this non-terminal).
The depth-first search (DFS) instead expands the \emph{lowest} node containing a non-terminal.
The DFS may not terminate, we need to introduce a stopping criterion; 
the most natural is the size of the program: we fix a target size $n$ and before applying a derivation rule we check whether the partial program has size at most $n$. With this stopping criterion the DFS enumerates all programs of size at most $n$.

The remaining algorithms take into account the probabilities attached to the PCFG.

\paragraph*{The sort and add search~\cite{Balog2017}.}
The algorithm is based on two ideas: the first idea is called `biassing'; we specify the order on the derivation rules for each non-terminal,
naturally by non-increasing order of probabilities, 
and the second idea is an incremental search: we fix a number $k$, restrict for each non-terminal to the $k$ most likely derivation rules,
runs the search on this restricted grammar, and iterates with a larger $k$ in case the program was not found.
The choice of the sequence of values for $k$ is both crucial for the performances and very hard to make.
In our experiments we use arithmetic progressions.

\paragraph*{The maximal probability from a non-terminal.}
The following subroutine was introduced in~\cite{Menon2013} and will be used by all subsequent algorithms.
For each non-terminal $T$, we write $m_T$ for the most likely program generated from $T$, and $p_T$ for its probability.
Computing $m_T$ and $p_T$ is done using dynamic programming:
$p_T$ is the maximum of $p \cdot \prod_{i \in [1,k]} p_{T_i}$ over all derivation rules $T \to^p f(T_1,\dots,T_k)$.

\paragraph*{The threshold algorithm~\cite{Menon2013}.}
The algorithm fixes a threshold, enumerates all programs with probability beating the threshold, and iterates with a lower threshold.
Enumerating all programs beating a fixed threshold is achieved using a DFS with the following stopping criterion:
before expanding a partial program we check whether it may be completed into a program with probability beating the threshold
by replacing each non-terminal $T$ by the probability of the most likely program $p_T$
and checking whether that program has high enough probability.
In our experiments we use geometric progressions for the sequence of thresholds.

\paragraph*{The beam search~\cite{Zhang2018}.}
Beam search is a BFS algorithm with bounded memory: the size of the queue of partial program,
called the beam, is restricted to a constant, the beam width. 

\section{Heap Search}
\label{sec:appendix_heap_search}

The pseudocode is given in Algorithm~\ref{algo:heap_search}.
\begin{algorithm*}
  \caption{Heap search}
  \label{algo:heap_search}
  \begin{algorithmic}[1]
    \Procedure{Initialization}{ }
    \ForAll{non-terminal symbols $T$} 
    \State Create an empty max heap $\heap_T$
    \State Create an empty hash table $\succ_T$
    \State Create an empty set $\seen_T$
    \ForAll{derivation rules $T \to f(T_1,\dots,T_k)$ }
    \State Add $f(m_1, \dots, m_k)$ to $\heap_T$ with priority $\D(f(m_1, \dots, m_k))$
    \State Add $f(m_1, \dots, m_k)$ to $\seen_T$
    \EndFor
    \EndFor
    \EndProcedure
    \State
    \Procedure{Query}{$T$,$x$}
    \If {$x$ is a key of $\succ_T$}
    \State Return $\succ_T[x]$
    \Else
    \State $x' \leftarrow pop(\heap_T)$ \algorithmiccomment{$x'$ is the successor of $x$}
    \State $\succ_T[x] \leftarrow x'$ \algorithmiccomment{updating the data structure}
    \State{(Assumes $x' = f(x_1,\dots,x_k)$ is generated by $T \to f(T_1,\dots,T_k)$)}
    \ForAll {$i \in [1,k]$} \algorithmiccomment{add all potential successors}
    \State $y_i = \query(T_i,x_i)$
    \State $x'_i = f(x_1,\dots, x_{i-1}, y_i, x_{i+1}, \dots, x_k)$
    \If {$x'_i$ is not in $\seen_T$}
    \State Add $x'_i$ to $\heap_T$ with priority $\D(x'_i)$
    \State Add $x'_i$ to $\seen_T$
    \EndIf
    \EndFor
    \State Return $x'$
    \EndIf
    \EndProcedure
\end{algorithmic}
\end{algorithm*}

\subsection{Correctness proof}

The following lemma gives the correctness of the heap search algorithm.
\begin{lemma}
  For any non-terminal $T$ and program $x$ generated from $T$, if we
  have already run \query(T,y) for any program y preceding x (among
  those generated from T) then \query(T,x) outputs the successor of
  $x$ from $T$
\end{lemma}
\begin{proof}
  Suppose the lemma is not true and consider the first time where it
  fails, on say, input $(T, x)$. At step 14, $(T, x)$ was not queried
  before and therefore the algorithm branches out to line 16. Saying
  that the algorithm fails is equivalent to say that line 17 fails,
  meaning that $x' = pop(\heap_T)$ is not the successor of $x$ from
  $T$. Let us denote by $y$ the correct successor. There are two cases
  which can lead the algorithm the failure:
  \begin{itemize}
  \item \textbf{First case: }$x'$ is a program with higher
    probability than $y$. In this case, it means that $x'$ was already
    the output of a previous query (because $(T,x)$ is the first query
    for which the algorithm fails, so programs with higher
    probability than $y$ have already been output before). Thus the
    program $x'$ was pop at least two times from $\heap_T$, one time
    when $x'$ was the correct output of a query, and another time when
    the program failed: this is not possible since we push programs
    only once into the heaps thanks to the condition at line 22.
  \item \textbf{Second case: } $x'$ has a smaller probability than the
    probability of $y$. In this case, it means that $y$ was not in
    $\heap_T$ when the algorithm performed a pop at line 17 (otherwise
    $y$ would have pop since $y$ has higher probability than
    $x'$). Suppose $y = f(x_1, \dots, x_k)$, generated by the
    derivation rule $T \to f(T_1,\dots,T_k)$. If all $x_i$ are maximal
    (meaning that they don't have a predecessor from $T_i$) then
    $f(x_1,\dots, x_k)$ is pushed in $\heap_T$ during the
    initialization procedure (line 6,7) so the algorithm cannot fail
    because of this case. Therefore there is at least one $x_i$, say
    w.l.o.g $x_1$, which has a predecessor $x'_1$ from $T_1$. Consider
    the program $f(x'_1,x_2,\dots, x_k)$; this program has higher
    probability than $y$ and therefore has been seen before. Thus
    $f(\query(T_1,x'_1), x_2, \dots, x_k)$ was added to $\heap_T$
    because of line 23. To conclude, observe that
    $f(\query(T_1,x'_1), x_2, \dots, x_k) = f(x_1, x_2, \dots, x_k))$
    so $y$ has previously been added to $\heap_T$.
  \end{itemize}
\end{proof}

\subsection{Complexity analysis}

\begin{lemma}
  Fix any non-terminal $T$. Suppose that we have already generated the
  first i programs generated from $T$ (meaning that if $x$ is the
  $j$-th program for $j< i$, then $x$ is a key of the hash table
  $\succ_T$ and $\succ_T[x]$ is the $(j+1)$-th program generated from
  $T$). Then querying the successor of the $i$-th program has a running
  time of $O(\log i)$.
\end{lemma}

\begin{proof}
  First observe that a query can call recursively several others
  queries. However, for any non-terminal symbol $T$ there is at most
  one query of the form $\query(T,x)$ which leads the algorithm to
  branch out to line 16 and thus to possibly rise other
  queries. Indeed, this case happens only when the successor of $x$
  has not been computed yet (otherwise the query stops at line 15);
  this can happen for at most one program for any fixed symbol $T$:
  the last program from $T$ already seen in any execution of the
  algorithm. Forgetting about recursive queries, the running time of a
  query going through line 16 is given by the pop and push operations
  (line 17 and 23). The number of pops and pushs is at most $m+1$
  where $m$ is the maximal arity of a function in the
  grammar. Moreover, each such operation costs a running time of
  $O(\log |\heap_T|)$, so the total time for the query is
  $O(\log |\heap_T|)$.

  Overall, the total running time is bounded by

  $$ \sum_{T} O(\log |\heap_T|).$$
  To conclude, observe that when we query the successor of the i-th
  program generated from $T$, the size of $\heap_T$ is bounded by
  $m\cdot i$ since we push at most $m$ programs during any query not
  already computed before.
\end{proof}

\subsection{Program evaluation}
A key advantage of \textsc{Heap Search} is that it is bottom-up: 
partial programs are composed from the leaves to the root.
This implies that partial programs can be evaluated and their evaluations cached.
Although memory hungry, this optimisation leads to major gains when taking evaluation into account.

\section{SQRT Sampling}
\label{sec:appendix_SQRT_sampling}
\subsection{An example}
The simplest instantiation of this theorem is $X = \set{\textrm{Head},\textrm{Tail}}$ and the Bernoulli distribution $\D$ with parameter $p$:
it draws Head with probability $p$ and Tail with probability $1-p$.

A sampling algorithm $\D'$ is a Bernoulli distribution with parameter $p'$,
and its loss is
\[
\L(\D',\D) = \E_{x \sim \D} \left[ \frac{1}{\D'(x)} \right] = \frac{p}{p'} + \frac{1-p}{1-p'}.
\]
Minimising for $p'$ yields $p' = \frac{\sqrt{p}}{\sqrt{p} + \sqrt{1 - p}}$,
inducing the loss optimal sampling algorithm $\sqrt{\D}$.

\subsection{Implementation details}

The construction follows two steps: first construct a weighted CFG that recognises $\sqrt{\D}$, 
and then normalise it into a PCFG using~\cite{chi-1999-statistical}.
The normalisation requires computing the partition function $Z$ defined by 
\[
Z(S) = \sum_{P \text{ generated from } S} \D(P).
\]
In general the partition function can be computed by solving a system of polynomial equations.
This is easier in our case since we restrict ourselves to acyclic PCFGs.

\section{Parallel implementation}
\label{sec:appendix_parallel_implementation}
\subsection{Description and pseudocode}

The pseudocode of the grammar splitter is given in Algorithm~\ref{algo:grammar_splitter} 
using two procedures: split and find improving swap.
The split procedure describes at a high level how our grammar splitter works.
The find improving swap procedure is here to provide a clear method of finding an improving swap or refinement.

In our experiments, we initialized the splitting as follows: 
we split the node with highest probability until the total number of nodes is greater than the number of splits required, 
then assign one node to each split, and the remaining nodes to the last split.
We also limited the search of an improving swap or refinement to the most probable split and the least probable split unlike in the find improving swap procedure where all splits are considered.
Finally, in all of our experiments $\alpha_{desired}=1.05$.

\newcommand{\splits}{\textsc{Splits}}
\begin{algorithm*}
  \caption{Grammar splitter}\label{algo:grammar_splitter}
  \begin{algorithmic}[1]
   
    \Procedure{split}{$G$, $nsplits$, $\alpha_{desired}$}
    \State Create an initial splitting $\splits$
    \State $\alpha \leftarrow \frac{\max_{sG \in \splits}{\text{probability mass}(sG)}}{\min_{sG \in \splits}{\text{probability mass}(sG)}}$

    \While {$\alpha > \alpha_{desired}$}
      \If {an improving swap exists}
      \State Update $\splits$ with the improving swap
      \State $\alpha \leftarrow \frac{\max_{sG \in \splits}{\text{probability mass}(sG)}}{\min_{sG \in \splits}{\text{probability mass}(sG)}}$
      \Else
      \State Find the partial program with largest probability $P$
      \State Replace $P$ in its split with its children
      \EndIf
    \EndWhile

    \State Return $\splits$
    \EndProcedure

    \Procedure{Find improving swap}{$G$, $\splits$}
    \State $\alpha \leftarrow \frac{\max_{sG \in \splits}{\text{probability mass}(sG)}}{\min_{sG \in \splits}{\text{probability mass}(sG)}}$
    \State $\alpha^* \leftarrow \alpha$ \algorithmiccomment{best improving swap $\alpha$}
    \State $s \leftarrow$ None \algorithmiccomment{best improving swap}
    \State $L \leftarrow \text{argmax}_{sG \in \splits}\ \text{probability mass}(sG)$
    \State Sort $\splits$ by increasing $\text{probability mass}$
    \ForAll{$sG \in \splits \setminus \{L\}$}
      \ForAll{$P' \in L$}
        \ForAll{$P \in G$}
          \State $\beta \leftarrow$ Compute new $\alpha$ with Swap($L$, $sG$, $P$, $P'$)
          \If{$\beta < \alpha^*$}
            \State $\alpha^* \leftarrow \beta$ 
            \State $s \leftarrow$ Swap($L$, $sG$, $P$, $P'$)
          \EndIf
        \EndFor
        \State $\beta \leftarrow$ Compute new $\alpha$ with Gift($L$, $sG$, $P'$)
        \If{$\beta < \alpha^*$}
          \State $\alpha^* \leftarrow \beta$ 
          \State $s \leftarrow$ Gift($L$, $sG$, $P'$)
        \EndIf
      \EndFor
    \EndFor
    
    \State $l \leftarrow \text{argmin}_{sG \in \splits}\ \text{probability mass}(sG)$
    \State Sort $\splits$ by decreasing $\text{probability mass}$
    \ForAll{$sG \in \splits \setminus \{L, l\}$}
      \ForAll{$P \in sG$}
        \ForAll{$P' \in l$}
          \State $\beta \leftarrow$ Compute new $\alpha$ with Swap($l$, $sG$, $P$, $P'$)
          \If{$\beta < \alpha^*$}
            \State $\alpha^* \leftarrow \beta$ 
            \State $s \leftarrow$ Swap($l$, $sG$, $P$, $P'$)
          \EndIf
        \EndFor
          \State $\beta \leftarrow$ Compute new $\alpha$ with Gift($sG$, $l$, $P$)
          \If{$\beta < \alpha^*$}
            \State $\alpha^* \leftarrow \beta$ 
            \State $s \leftarrow$ Gift($sG$, $l$, $P$)
          \EndIf
      \EndFor
    \EndFor
    \State Return $s$
    \EndProcedure

\end{algorithmic}
\end{algorithm*}

\section{Details on the experiments}
\label{sec:appendix_experiments}

\subsection{Random PCFG search}

To turn the CFGs into PCFGs we sample a probabilistic labelling, meaning
a weight for each derivation rule such that the sum over all derivation rules from each non-terminal is one.
The distribution depends on a parameter $\alpha \in (0,1]$: 
the $i$\textsuperscript{th} weight is sampled uniformly at random in $[0,\alpha^i]$,
and eventually renormalised.
The smaller $\alpha$, the more biassed the distribution, 
implying that the search for programs will be faster since we have better hints on the target program. 
For $\alpha = 1$ the weights are sampled uniformly, resulting in a very unbiased PCFG, making the search more difficult.
In the random PCFGs experiment we use $\alpha = 0.7$.

\subsection{Machine-learned PCFG}

We only work with tasks of type \verb+int list -> int list+.
We remove examples where lists have length greater than $L_{max} = 10$ or where one of the elements of the input or output list are not in $L_{in} = \left[-30; 30\right]$.

We say that a task is solved if a program is found which satisfies all examples.
The timeout of $100$s only takes into account the search time and the evaluation times and not the time to query the neural network for predictions.

\subsubsection*{The neural network}

The neural network takes as input (the encoding of) a list of examples and outputs a probabilistic labelling for the CFG.

\paragraph{Input encoding}
Each list in the examples is encoded in a naive way by mapping each element of $L_{in}$ to $[0,60]$.
We additionally use two special symbols \verb+PAD+ and \verb+NOPAD+, leading to 
a fixed size encoding of a list into a vector of size $2 L_{max}$.
Hence an example is encoded as a tensor of shape $(n_{inputs}, 2 L_{max})$ where $n_{inputs}$ is the number of inputs in the example.

\paragraph{Embedding}
The encoded examples are fed into an embedding consisting of a single layer GRU~\cite{cho-etal-2014-learning} outputting a tensor of shape $(s_{GRU} \times 2L_{max})$.
In our experiments $s_{GRU}=10$.

\paragraph{Main layers}
The embedded examples are fed into an MLP with 3 layers.
The first two have output size $s_{MLP} = 64$ and the last layer outputs a tensor of dimension $k$
the number of derivation rules in the CFG $C$.
Assigning these (normalised) weights to the rules of the CFG yields a PCFG.

\subsubsection*{Training}

We optimised end-to-end the neural network with Adam~\cite{Kingma2015AdamAM} with default parameters and learning rate of $lr=0.001$.
We trained for one epoch with a batch size of $128$ on a generated dataset of $10,000$ problems.
To generate the dataset, programs are sampled from the uniform PCFG and inputs by choosing a length at most $L_{max}$
and elements uniformly at random in $L_{in}$.
If the output of such a generated input does not fall in the lexicon $L_{in}$ then the program is discarded.

We use the binary cross entropy loss between the output of the neural network and the encoding of a solution program: 
a rule in the CFG has probability 1 if it is used to derive the solution program, and 0 otherwise.

\nocite{BunelHDSK18,ClymoGPFM20}

\end{document}